\documentclass[12pt]{l4dc2023}

\title[Spectrally Normalized Memory Networks]{Computationally Light Spectrally Normalized Memory Neuron Network based Estimator for GPS-denied operation of Micro-UAV}
\usepackage{times}

\usepackage{empheq}
\usepackage{amsfonts}
\usepackage{bbold}
\usepackage{booktabs}
\usepackage{tikz}
\usepackage{siunitx}
\usepackage{hyperref}
\graphicspath{ {./images/} }
\usepackage[english]{babel}
\usepackage{pgfplots}
\usepackage{balance}
\usepackage{xcolor}

\pgfplotsset{compat=1.17}

\DeclareMathOperator*{\argmin}{\arg\!\min}
\DeclareMathSymbol{\shortminus}{\mathbin}{AMSa}{"39}
\newcommand*\widefbox[1]{\fbox{\hspace{2em}#1\hspace{2em}}}

\tikzset{%
	every neuron/.style={
		circle, 
		draw,
		minimum size = 0.7cm
	},
	neuron memory/.style={
		circle,
		draw,
		minimum size = 0.5cm, 
		fill=black
	},
	neuron missing/.style={
		draw=none,
		scale=2.5,
		text height=0.133cm,
		execute at begin node=\color{black}$\vdots$
	},
}

\author{%
 \Name{Nishanth Rao} \Email{nishanthrao@iisc.ac.in}\\
 \addr Indian Institute of Science, Bengaluru
 \AND
 \Name{Suresh Sundaram} \Email{vssuresh@iisc.ac.in}\\
 \addr Indian Institute of Science, Bengaluru
 \AND
 \Name{Varun Raghavendra} \Email{varuncr@iisc.ac.in}\\
 \addr Indian Institute of Science, Bengaluru
}

\begin{document}

\maketitle

\begin{abstract}%
 This paper addresses the problem of position estimation in UAVs operating in a cluttered environment where GPS information is unavailable. A learning-based approach is proposed that takes in the rotor RPMs and past state as input and predicts the one-step-ahead position of the UAV using a novel spectral-normalized memory neural network (SN-MNN). The spectral normalization guarantees stable and reliable prediction performance. The predicted position is transformed to the global coordinate frame (GPS), which is then fused along with the odometry of other peripheral sensors like IMU, barometer, compass, etc., using the onboard extended Kalman filter (EKF) to estimate the states of the UAV. The experimental flight data collected from an RTK-GPS facility using a micro-UAV is used to train the SN-MNN. The \texttt{PX4-ECL} library is used to fuse the predicted data using the SN-MNN, and the estimated position is compared with actual ground truth data. The proposed algorithm doesn't require any additional onboard sensors and is computationally light. The performance of the proposed approach is compared with the current state-of-art GPS-denied algorithms, and it can be seen that the proposed algorithm has the least RMSE for position estimates.
\end{abstract}

\section{Introduction}

Advancements in UAV technology have enabled their widespread usage in logistic transportation, urban air mobility, and agriculture. A crucial aspect of the UAV flight is the accuracy of the onboard navigation system that provides a sense of whereabouts to the UAV controller. The onboard navigation system relies heavily on GPS sensors that accurately estimate the position of the UAV. However, in cluttered environments like forests and indoor environments, there is an intermittent loss of GPS (sometimes no GPS signal), which can lead to inaccurate position estimates, rendering the UAV unstable. Thus, it is crucial to look for GPS-denied alternatives to provide reliable position information to the UAV.

The existing literature on algorithms developed for GPS-denied operation can be broadly divided into two categories: algorithms that use reliable position estimates generated by either vision-based systems or simultaneous localization and mapping (SLAM)-based systems. In \citep{mohta2018fast}, stereo camera-based localization is employed for fast and agile navigation of the UAV in the presence of surrounding obstacles. In \citep{mebarki2014image}, the concept of \textit{image moments} is utilized from environment images to estimate the translational velocity of the UAV reliably. A combination of optical flow and ultrasound sensors has been used in \citep{tsai2016optical} to achieve a stable indoor hovering performance of a micro aerial vehicle. However, vision-based localization systems fail to provide reliable odometry in varying lighting conditions and exhibit drift when no features are available to track. Further, one needs to tune the parameters meticulously to achieve reliable performance in a known environment\citep{balamurugan2016survey}.

On the other hand, SLAM-based methods mostly require LiDAR sensors or compatible vision-based sensors. In \citep{kim20056dof}, a 6-DOF SLAM is proposed that builds a relative map from surrounding features and provides odometry measurements relative to this map. In \citep{urzua2017vision}, a vision-based SLAM algorithm is employed using a monocular camera and an ultrasound camera for navigation. To alleviate the problems specific to vision-based sensors, \citep{khattak2019robust} employs an infrared thermal sensor to obtain localization information in the presence of dark, texture-less environments with dust-filled / smoke-filled settings. However, SLAM-based methods work efficiently only in a cluttered environment with some external objects that can be tracked throughout the flight. When flying in areas without any surrounding objects or in fast dynamic environments, vision-based odometry and SLAM-based methods fail to provide accurate position estimates to the UAV. Reliable vision-based odometry systems require high FPS performance cameras that are prohibitively expensive. Moreover, these methods require additional onboard specialized sensors that can be limiting in a micro-UAV setting and tend to consume extra power that can limit the overall endurance of the UAV.

In this paper, a \textit{data-driven model learning}-based approach is proposed to estimate the UAV position using a Spectrally Normalized Memory Neuron Network (SN-MNN) that is invariant to the environmental features and external factors like surrounding objects, varying lighting conditions, etc. The SN-MNN predicts the position of the UAV based on the rotor RPM input and previous UAV states. It is shown theoretically that spectral normalization guarantees a stable prediction performance by constraining the Lipschitz constant of the fitted function. The look-ahead predicted position is transformed to a global coordinate (GPS), and extended Kalman filter-based state fusion is used to estimate the UAV states. The experimental flight data from an RTK-GPS facility is used for training the SN-MNN. The model learning-based approach is validated using the \texttt{PX4-ECL} library on sample test flights. Finally, the performance of the proposed algorithm is compared with other state-of-art GPS-denied algorithms. It can be seen that the proposed algorithm has the least RMSE in predicting the position of the UAV in comparison to other techniques.

\section{Spectral-Normalized Memory Neuron Network-based State Estimation}
\label{sec:section_2}
First, this section presents the UAV dynamic equations and input-output model. Next, the novel data-driven spectrally normalized memory neuron network is presented to predict the position of the UAV from the past state and current input. Finally, the predicted local position is converted to global coordinates (GPS) and state fusion is carried out to estimate the UAV states.

\subsection{Input-Output Model of UAV}
The physics-based mathematical model of a typical UAV system is given below:
\begin{subequations}
\begin{empheq}[box=\widefbox]{align}
    \dot{\pmb{x}} &= \pmb{v}, \ \ \ m\dot{\pmb{v}} = mg\hat{\pmb{k}} + \pmb{R}\hat{\pmb{k}}f_t + \Tilde{\pmb{f}} \\
    \dot{\pmb{R}} &= \pmb{R}\pmb{\Omega}^{\times}, \ \ \ \pmb{J}\dot{\pmb{\Omega}} + \pmb{\Omega}^{\times}\pmb{J\Omega} = \pmb{\tau} + \Tilde{\pmb{\tau}} \\
    \text{with}, \ \ f_t &= K_\omega\left(\omega_1^2 + \omega_2^2 + \omega_3^2 + \omega_4^2 \right)\label{eq:ft} \\
    \pmb{\tau} &= \left[\begin{array}{l} \tau_x \\ \tau_y \\ \tau_z \end{array} \right] = \left[\begin{array}{l} K_{\omega}l\left(\omega_3^2 - \omega_1^2\right) \\ K_{\omega}l\left(\omega_4^2 - \omega_2^2 \right) \\ K_d\left( \omega_2^2 + \omega_4^2 - \omega_1^2 - \omega_3^2 \right)  \end{array}\right]\label{eq:tau}
\end{empheq}
\label{eq:uav_dynamics}
\end{subequations}
where $\hat{\pmb{k}}=\left[0 \ 0 \ 1 \right]^T$ is the unit vector along the $z-$axis,  $\pmb{x}\in\mathbb{R}^3$ is the position of the UAV with mass $m\in\mathbb{R}$ and moment of inertia $\pmb{J}\in\mathbb{R}^{3\times3}$, $\pmb{v}\in\mathbb{R}^3$ is the linear velocity, $\pmb{R}\in\mathbb{R}^{3\times3}$ is the rotation matrix that converts a vector from the UAV-fixed body frame to the inertial frame, $\pmb{\Omega}\in\mathbb{R}^3$ is the angular velocity of the UAV, $f_t\in\mathbb{R}$ is the input thrust vector and $\pmb{\tau}\in\mathbb{R}^3$ is the input torque vector given by Eq. (\ref{eq:ft}) and (\ref{eq:tau}) respectively. The quantities $\omega_i$ denotes the rotational velocity of the $i^{th}$ motor in (rad/s), $l$ denotes the arm length of the UAV, and the constants $K_\omega, K_d$ denote the motor constant and the drag coefficient respectively. The operator $(.)^{\times}:\mathbb{R}^3\rightarrow\mathbb{R}^{3\times3}$ is the hat operator that converts a vector to a skew-symmetric matrix. The quantities $\Tilde{\pmb{f}}, \Tilde{\pmb{\tau}}$ represent the external complex aerodynamic forces and torques on the UAV that are mostly unknown or cannot be analytically modelled. In real world, the dynamics given by Eq. (\ref{eq:uav_dynamics}) are rudimentary and cannot be relied upon for estimating the UAV states due to presence of noise, model uncertainties and external disturbances.

In general, one can use \textit{billings theorem} \citep{leontaritis1985input} to write the input-output model of a dynamical system as:
\begin{equation}
    \pmb{y}_{k+1} = f\left(\pmb{y}_k,\hdots,\pmb{y}_{k-n},\pmb{u}_k,\hdots, \pmb{u}_{k-n}\right)
\end{equation}
where $f(.)$ is an unknown nonlinear function, $n$ is the order of the system, $\pmb{y}_k$ and $\pmb{u}_k$ is the output of and input to the system respectively at time step $k$. Note that one can use a recurrent neural network to approximate the unknown nonlinear function using the current input and past output. It has been shown in the literature that the \textit{Memory Neuron Network} (MNN) \cite{sastry1994memory} is more efficient in approximating the dynamics accurately than other state-of-the-art recurrent neural networks \citep{9534209}. The presence of uncertainty in thrust and the unknown disturbance influences the stability/reliability of prediction. The next section proposes a spectrally normalized MNN to learn the UAV model accurately.

\subsection{Spectral-Normalized Memory Neural Network based Model Learning}

\begin{figure}
    \begin{tikzpicture}[scale=0.9, every node/.style={transform shape}]
    
        \draw [->] (-1.0, 4.25) -- (0.0, 4.25);
        \node at (-0.55, 4.5) {$\pmb{\omega}$};
        
        \draw [->] (-0.5, 4.25) -- (-0.5, 3);
        \filldraw[fill=black!20] (-0.25, 3.0) -- (-0.75, 3.0) -- (-0.5, 2.5) -- cycle;
        \node at (-0.9, 2.75) {$\frac{1}{\omega_m}$};
        \draw [->]  (-0.5, 2.5) -- (-0.5, 0.6) -- (1.0, 0.6);
        \node at (0.6, 0.8) {$\overline{\pmb{\omega}}$};
        
        \draw[densely dashed, ->] (3.75, 3.5) -- (3.75, 2.8) -- (0.3, 2.8) -- (0.3, 1.3) -- (1.0, 1.3);
        \node at (0.6, 1.5) {$\pmb{p}$};
        
        \draw (1.0,0) rectangle (3.0, 2);
        \node at (2.0, 1.6) {SN-MNN};
        \node at (1.9, .7) {
        \begin{tikzpicture}[x=1.5cm, y=1.5cm, >=stealth, scale=0.25, every node/.style={transform shape}, opacity=0.2]
        
        \foreach \m [count=\i] in {1, memory, missing, 4, memory}
            \node [every neuron/.try, neuron \m/.try](input-\m-\i) at (1.0, 2.5-\i/2) {};
        
        \foreach \m [count=\i] in {1, memory, 3, memory, missing, 6, memory}
        	\node [every neuron/.try, neuron \m/.try](hidden-\m-\i) at (2.4, 2.8-\i/2) {};
        
        \foreach \i in {1, 2}
        	\node [every neuron/.try](output-\i) at (3.8,2.5-\i) {};
        
        \foreach \i in {1, 2}
        	\node [neuron memory/.try](output-memory-\i) at (3.8, 2.0-\i) {};
        
        \foreach \i in {1, 2}
        	\draw[->] (output-\i) -- (output-memory-\i);
        	
        \foreach \i in {1, 4}
        	\foreach \j in {1, 3, 6}
        		\draw [->] (input-\i-\i) -- (hidden-\j-\j);

        \foreach \i in {1, 3, 6}
        	\foreach \j in {1, 2}
        		\draw [->] (hidden-\i-\i) -- (output-\j);
        		
	\end{tikzpicture}
	};
	
	\draw (0.0, 3.5) rectangle (2.5, 5);
	\node at (1.25, 4.7) {Real World};
	\node at (1.25, 4.3) {UAV-Dynamics};
	\node at (1.25, 3.9) {$\dot{\pmb{x}} = f(\pmb{x}, \pmb{\omega})$};
	
	\draw[->] (2.5, 4.25) -- (3, 4.25);
	
	\draw (3.0, 3.5) rectangle (4.5, 5);
	\node at (3.75, 4.5) {Sensor};
	\node at (3.75, 4.0) {$\pmb{y} = h(\pmb{x})$};
	\draw[->] (4.5, 4.25) -- (5., 4.25); 
    
    \node at (4.85, 4.4) {\tiny$+$};
    \node[circle, draw, minimum size=0.5cm] at (5.25, 4.25) {}; 
    \node at (5.1, 3.9) {-};
    
    \draw[->] (3.0, 1) -- (5.25, 1) -- (5.25, 4);
    \node at (3.2, 1.3) {$\hat{\pmb{y}}$};
    
    \draw[->] (5.5, 4.25) -- (6, 4.25);
    \node at (5.75, 4.4) {$\pmb{e}$};
    
    \end{tikzpicture}
    \hspace{2cm}
    \begin{tikzpicture}[x=1.5cm, y=1.5cm, >=stealth, scale=0.6, every node/.style={transform shape}]
    
        \node [every neuron](nn-demo) at (0.8, 4.2) {};
        \draw [<-] (nn-demo) -- ++(-0.5, 0);
        \draw [->] (nn-demo) -- ++(0.5, 0);
        \draw (0.65, 3.6) rectangle (0.95, 3.9) node[pos=.5]{$z^{\shortminus 1}$};
        \node [neuron memory](mnn-demo) at (0.8, 3.3) {};
        \draw [->] (mnn-demo) -- ++(0.5, 0);
        \draw (0.65, 2.75) rectangle (0.95, 3.05) node[pos=.5]{$z^{\shortminus 1}$};
        \draw [->] (nn-demo.east) .. controls +(left:-4mm) and +(right:4mm) .. (0.95, 3.75) node at (1.3, 3.85) {\scriptsize $\alpha_1^i$};
        \draw [->] (0.65, 3.75) .. controls +(left:4mm) and +(right:-4mm) .. (mnn-demo.west);
        \draw [->] (mnn-demo.east) .. controls +(left:-4mm) and +(right:4mm) .. (0.95, 2.9) node at (1.4, 2.95) {\scriptsize $1\shortminus\alpha_1^i$};
        \draw [->] (0.65, 2.9) .. controls +(left:4mm) and +(right:-4mm) .. (mnn-demo.west);
        
        \draw [dashed] (0.12, 2.7) -- (1.65, 2.7) -- (1.65, 4.5) -- (0.12, 4.5) -- cycle;
        \draw [dashed] (0.12, 2.7) -- (0.7, 2.3);
        \draw [dashed] (1.65, 2.7) -- (1.3,  2.4);
        
        \node [every neuron](nn-legend) at (3.0-0.5, 4.0) {};
        \node [neuron memory](mnn-legend) at (3.0-0.5, 3.4) {};
        \node at (4.2-0.5, 4.0) {Network neuron}; 
        \node at (4.2-0.5, 3.4) {Memory neuron};
        
        \draw (2.5-0.5, 4.4) -- (5.2-0.5, 4.4) -- (5.2-0.5, 3.0) -- (2.5-0.5, 3.0) -- cycle;
        
        
        \foreach \m [count=\i] in {1, memory, missing, 4, memory}
            \node [every neuron/.try, neuron \m/.try](input-\m-\i) at (1.0, 2.5-\i/2) {};
        
        \foreach \i\j in {1, 4}
        	\draw [<-] (input-\i-\i) -- ++(-0.5, 0) {};
        		
        \foreach \i\j in {1/2, 4/5}
        	\draw[->] (input-\i-\i) -- (input-memory-\j);
        
        \foreach \m [count=\i] in {1, memory, 3, memory, missing, 6, memory}
        	\node [every neuron/.try, neuron \m/.try](hidden-\m-\i) at (2.4, 2.8-\i/2) {};
        
        \foreach \i\j in {1/2, 3/4, 6/7}
        	\draw[->] (hidden-\i-\i) -- (hidden-memory-\j);
        
        
        \foreach \i in {1, 2}
        	\node [every neuron/.try](output-\i) at (3.8,2.5-\i) {};
        
        \foreach \i in {1, 2}
        	\node [neuron memory/.try](output-memory-\i) at (3.8, 2.0-\i) {};
        
        \foreach \i in {1, 2}
        	\draw[->] (output-\i) -- (output-memory-\i);
        	
        \foreach \i in {1, 4}
        	\foreach \j in {1, 3, 6}
        		\draw [->] (input-\i-\i) -- (hidden-\j-\j);
        
        \foreach \i in {2, 5}
        	\foreach \j in {1, 3, 6}
        		\draw [->] (input-memory-\i) -- (hidden-\j-\j);
        
        \foreach \i in {1, 3, 6}
        	\foreach \j in {1, 2}
        		\draw [->] (hidden-\i-\i) -- (output-\j);
        
        \foreach \i in {2, 4, 7}
        	\foreach \j in {1, 2}
        		\draw [->] (hidden-memory-\i) -- (output-\j);
        
        \foreach \i in {1, 2}
        	\draw [->] (output-memory-\i.west) .. controls +(left:6mm) and +(right:-6mm) .. (output-\i.west);
        	
        \foreach \l [count=\i] in {x, y}
        	\draw [->] (output-\i) -- ++(0.5, 0) {};
        
        \node at (1.6, 2.25) {\scriptsize $n_1^i$};
        \node at (1.6, 1.95) {\scriptsize $r_1^i$};
        \draw [->] (1.8, 2.15) .. controls +(left:1mm) and +(right:-3mm) .. (2.0, 2.7);
        \draw [->] (1.85, 1.98) .. controls +(left:1mm) and +(right:-2mm) .. (2.0, 2.4);
        
        \node at (2.17, 2.75) {\scriptsize $w_{11}^i$};
        \node at (2.06, 2.53) {\scriptsize $q_{11}^i$};
        
        \draw [->] (3.4, 0.09) .. controls +(left:2mm) and +(right:-3mm) .. (3.5, -0.4);
        \node at (3.6, -0.4) {\scriptsize $\beta_2^L$};
        
        \draw [dashed] (0.7, 0.8 + 0.5) -- (1.3, 0.8 + 0.5) -- (1.3, 1.8 + 0.5) -- (0.7, 1.8 + 0.5) -- cycle;
    \end{tikzpicture}

\caption{Figure on the left shows a schematic block diagram of SN-MNN training. The error is used for backpropagation. The motor input $\pmb{\omega}$ is normalized by the maximum rotor speed $\omega_m$ before feeding to the SN-MNN. The figure on the right shows the schematic diagram of a fully connected SN-MNN consisting of a single hidden layer.}
\label{fig:train_mnn}
\end{figure}
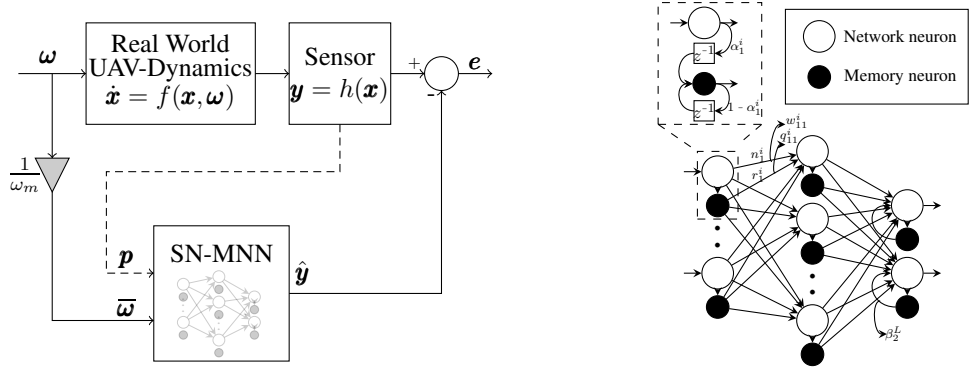

The Spectrally Normalized Memory Neural Network (SN-MNN) contains \textit{fully connected} network neurons (white circle) with its associated memory neurons (black solid circle). The unique nature of the connection between the Network Neurons and the Memory Neurons (see Fig.\ref{fig:train_mnn}) makes the network \textit{recurrent} in nature. The network is parameterized by $\pmb{\theta} = \left\{(^{1}\pmb{W}, ^{1}\pmb{Q}), \hdots, (^{L}\pmb{W}, ^{L}\pmb{Q})\right\}$, the weights corresponding to both the network neurons ($\pmb{W}$) as well as the memory neurons ($\pmb{Q}$). The left superscript denotes the layer number. Thus, the output of SN-MNN to an input $\pmb{p}$ can be compactly represented as:
\begin{align}
    f(\pmb{p}, \pmb{\theta}) = ^{L}\pmb{W}\left(\hdots \phi\left(^2\pmb{W}\left(\phi\left(^1\pmb{W}\pmb{p} + ^1\pmb{Q} ^{1}\pmb{r} \right)\right) + ^{2}\pmb{Q}^{2}\pmb{r}\right) \hdots\right) + ^{L}\pmb{Q}^{L}\pmb{r} \label{eq:MNN_output}
\end{align}
where $\phi(.)$ denotes the activation function and $^{l}\pmb{r}$ denotes the output of the memory neurons present in the $l^{th}$ layer. The recurrence relationship between the memory neurons and the network neurons can be represented as:
\begin{align}
    ^{l}\pmb{r}_{k} = ^{l}\pmb{\alpha} ^{l}\pmb{n}_{k-1} + (1 - ^{l}\pmb{\alpha})^{l}\pmb{r}_{k-1}
    \label{eq:v_and_n}
\end{align}
where $^{l}\pmb{n} = \phi(.)$ is the output of the activation function in the $l$\textsuperscript{th} layer and $^{l}\pmb{\alpha}$ is the weight of the feedback connections between the network and the memory neurons in the $l$\textsuperscript{th} layer (see Fig. \ref{fig:train_mnn}).

The Lipschitz constant $\gamma$ of a real valued function $f:\mathbb{R}^n\rightarrow\mathbb{R}$ is defined mathematically as:
\begin{align}
    \left\Vert f\left(\pmb{p}_2) - f(\pmb{p}_1\right)\right\Vert_2 \leq \gamma \left\Vert \pmb{p}_2 - \pmb{p}_1 \right\Vert_2
\end{align}
The Lipschitz constant of a differentiable function is the \textit{maximum spectral norm} of its Jacobian over the function's domain: $\gamma = \text{sup}_{\pmb{p}} \ \rho(\nabla f(\pmb{p}))$, where $\rho(\pmb{A})$ denotes the spectral norm of the matrix $\pmb{A}$ which is defined as the square root of maximum eigenvalue of the matrix $A^{H}A$. As demonstrated in \citep{shi2019neural}, it is essential to limit the Lipschitz constant of a neural network to ensure stable reliable prediction performance that is comparable with the actual dynamics of the UAV. The following Theorem guarantees the Lipschitz constant of the SN-MNN:
\begin{theorem}
The Lipschitz constant of the entire spectrally normalized memory neuron network satisfies the inequality $\left\Vert f(\pmb{p}, \pmb{\theta}) \right\Vert_2 \leq \gamma$ under the spectral weight normalization:
\begin{align}
    \overline{\pmb{W}} = \left(\frac{\pmb{W}}{\rho(\pmb{W})}\right) \cdot \gamma ^ {\frac{1}{L}}, \ \ \ \ \ \ \ \overline{\pmb{Q}} = \left(\frac{\pmb{Q}}{\rho(\pmb{Q})}\right) \cdot \gamma ^ {\frac{1}{L}}\label{eq:SWN}
\end{align}
with $\gamma$ being the intended Lipschitz constant of the network, and $tanh(.)$ as the activation function.
\end{theorem}
\begin{proof}
The spectral norm of of a linear map $g(\pmb{p}) = \pmb{Wp} + \pmb{b}$ can be simplified as: $\text{sup}_{\pmb{p}} \rho(\nabla g) = \text{sup}_{\pmb{p}} \rho(\pmb{W}) = \rho(\pmb{W})$. Moreover, using the inequality $\text{Lip}\left( g_{1} \circ g_{2}\right) \leq \text{Lip}\left(g_1\right) \cdot \text{Lip}\left( g_2\right)$ and the fact that $\text{Lip}\left(tanh(.)\right) = 1$ along with Eq. (\ref{eq:MNN_output}) leads to:

\begin{align}
    \left\Vert f\left(\pmb{p}, \pmb{\theta} \right) \right\Vert_2 &= \text{Lip}\left( ^{L}\overline{\pmb{W}}\cdot\left(\hdots \phi\left(^2\overline{\pmb{W}}\cdot\left(\phi\left(^1\overline{\pmb{W}}\pmb{p} + ^1\overline{\pmb{Q}} ^{1}\pmb{r} \right)\right) + ^{2}\overline{\pmb{Q}}^{2}\pmb{r}\right) \hdots\right) + ^{L}\overline{\pmb{Q}}^{L}\pmb{r} \right) \\
    & \leq \prod_{l=1}^{L} \rho(\overline{\pmb{W}}) = \prod_{l=1}^{L} \gamma^{\frac{1}{L}} = \gamma
\end{align}
Here, the term $^{i}\pmb{\overline{\pmb{Q}}}^{i}\pmb{r}$ can be considered as a "time-varying" bias that is independent of input $\pmb{p}_k$ at current time step $k$, and thus, doesn't affect the Lipschitz constant. However, the term $\pmb{r}$ depends on $\pmb{p}_{k-1}$ through Eq. \ref{eq:v_and_n} which is why the weight matrix $\pmb{Q}$ corresponding to the memory neurons must also undergo spectral normalization.
\end{proof}

For training the network, the \textit{modified backpropagation} approach as described in \citep{sastry1994memory} is used, with the following cost function being minimized during training at every time step:
\begin{subequations}
\begin{align}
    \pmb{\theta}^{*} &= \argmin_{\pmb{\theta}} \sum_{k=1}^{N} \frac{1}{T} \left\Vert \pmb{y}_{k} - f(\pmb{p}_k, \pmb{\theta}) \right\Vert_{2}^{2} \\
    \text{such that,} & \ \ \ \ \ \Vert f(\pmb{p}_{k}, \pmb{\theta}) \Vert_2 \leq \gamma\label{eq:gamma}
\end{align}
\end{subequations}
where the input $\pmb{p}_k = \left[\pmb{y}_{k-1}^T \ \ \overline{\pmb{\omega}}_k^T \ \ \pmb{\Theta}_{k}^T \right]^T$ consists of the previous position of the UAV $\pmb{y}_{k-1}\in\mathbb{R}^3$, the roll-pitch-yaw orientation of the UAV $\pmb{\Theta}_k\in\mathbb{R}^3$ and the current normalized motor RPM $\pmb{\overline{\omega}}_k\in\mathbb{R}^4$. The training process of the network is illustrated in Fig. \ref{fig:train_mnn}. Let $\pmb{e}$ be the error vector. Due to the constraint of Eq. (\ref{eq:gamma}), the update rules for $\pmb{W}$ and $\pmb{Q}$ are as follows:
\begin{empheq}[box=\widefbox]{align}
    ^l\pmb{W}_{k+1} &= \frac{\gamma^{\frac{1}{L}}}{\rho(^l\pmb{W}_k)}\left( ^l\pmb{W}_k - \eta \cdot ^l\pmb{n}_k\pmb{e}^T \right) \\
    ^l\pmb{Q}_{k+1} &= \frac{\gamma^{\frac{1}{L}}}{\rho(^l\pmb{Q}_k)}\left( ^l\pmb{Q}_k - \eta \cdot ^l\pmb{r}_k\pmb{e}^T \right)
\end{empheq}
More details of the update rule derivation are provided in the supplementary material \cite{nish2022supp_mat}.

\subsection{GPS Conversion and State Estimation}

The network predicts the position of the UAV based on the rotor RPM input. This position estimate can be used during \textit{state fusion} typically performed by an onboard Extended Kalman Filter (EKF), in addition to the state information provided by other peripheral sensors like IMU, compass, magnetometer, airflow sensor etc. In this work, the position estimates given by the network are converted to \textit{GPS coordinates} $\pmb{\zeta}_k$, also known as geodetic coordinates (latitude $\phi$, longitude $\lambda$, altitude $\mathfrak{z}$) and transformed GPS coordinates are used by the EKF for state estimation. This is illustrated in Fig. \ref{fig:GPS_Est} and Fig. \ref{fig:block_diag}. Since many flight controllers offer out-of-box support for real-time GPS fusion, it has been adopted in this paper. Moreover, ground control software utilises GPS coordinates to visualise the UAV path and monitor its itinerary and the course of navigation.

\begin{figure}
\begin{tikzpicture}[scale=0.9, every node/.style={transform shape}]
\begin{axis}[
  view={150}{30},
  axis lines=center,
  width=6cm,height=5cm,
  xticklabels={,,}, yticklabels={,,}, zticklabels={,,},
  xmin=-5,xmax=5,ymin=-5,ymax=5,zmin=0,zmax=4,
  clip=false,
]
\addplot3 [only marks] coordinates {(2,2,2)};
\addplot3 [no marks,densely dashed] coordinates {(0,0,2) (2,0,2) (2,0,0)};
\addplot3 [no marks,densely dashed] coordinates {(0,0,2) (0,2,2) (2,2,2)};
\addplot3 [no marks,densely dashed] coordinates {(2,0,2) (2,2,2) (2, 2, 0)};
\addplot3 [no marks,densely dashed] coordinates {(2,0,0) (2,2,0) (0,2,0)};
\addplot3 [no marks,densely dashed] coordinates {(0,2,0) (0,2,2)};

\node [above right] at (axis cs: 5, 0, 0) {$x$};
\node [above right] at (axis cs: 0, 5, 0) {$y$};
\node [right] at (axis cs: 0, 0, 4) {$z$};
\end{axis}

\draw[-stealth] (7, 1.5, 0) -- (7.5, 1.5, 0);
\draw[-stealth] (7, 1.5, 0) -- (7, 2, 0);
\draw[-stealth] (7, 1.5, 0) -- (7, 1.5, 0.8);

\draw[fill=black] (6.4, 1.8) circle (1.5pt);
\node at (5.6, 2.2) {$\left( \varphi',\lambda' ,\mathfrak{z}' \right)$};

\node at (1.4, 1.9) {$\left( x', y', z'\right)$};

\shade[ball color = lightgray,
    opacity = 0.6
] (7,1.5,0) circle (1cm);

\draw[->, thick] (3.3, 2.7) .. controls (4.5, 3.2) .. (5.5, 2.7);

\end{tikzpicture}
\hspace{2cm}
\includegraphics[scale=0.13]{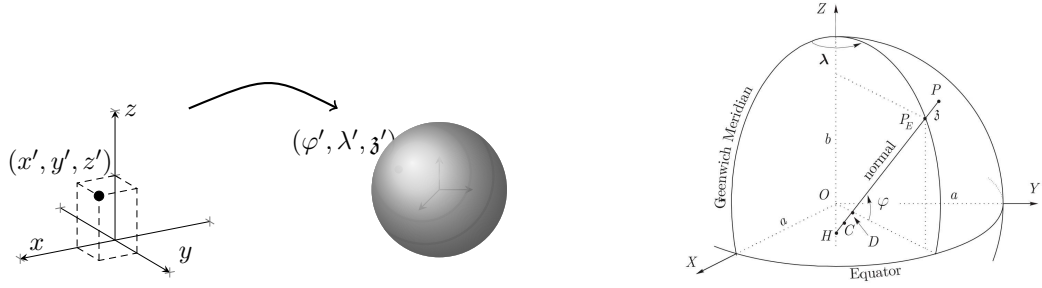}

\caption{Figure on the left shows the GPS Conversion from East-North-Up (ENU) coordinate vector $\pmb{x} = (x', y', z')$ shown in to geodetic (latitude, longitude and altitude) coordinate vector $\pmb{\zeta} = (\varphi', \lambda', \mathfrak{z}')$. Figure on the right shows the ECEF coordinate system and the geodetic coordinate system together.}
\label{fig:GPS_Est}
\end{figure}

\begin{figure}[!htb]
    \centering
    \hspace{-2.5cm}
    \begin{tikzpicture}[scale=0.75, every node/.style={transform shape}]
        
        
        \draw (-2.8, 3.75) rectangle (-1.0, 4.75);
        \node at (-1.9, 4.5) {PWM to};
        \node at (-1.9, 4.1) {RPM (ESC)};
        
        \draw[->] (-3.5, 4.25) -- (-2.8, 4.25);
        \node at (-3.6, 4.5) {RC Input};
        
        \draw [->] (-1.0, 4.25) -- (0.0, 4.25);
        \node at (-0.55, 4.5) {$\pmb{\omega}$};
        
        \draw [->] (-0.5, 4.25) -- (-0.5, 3);
        \filldraw[fill=black!20] (-0.25, 3.0) -- (-0.75, 3.0) -- (-0.5, 2.5) -- cycle;
        \node at (-0.9, 2.75) {$\frac{1}{\omega_m}$};
        \draw [->]  (-0.5, 2.5) -- (-0.5, 0.6) -- (1.0, 0.6);
        \node at (0.6, 0.8) {$\overline{\pmb{\omega}}$};
        
        \draw[densely dashed, ->] (3.75, 3.5) -- (3.75, 2.8) -- (0.3, 2.8) -- (0.3, 1.3) -- (1.0, 1.3);
        \node at (0.6, 1.5) {$\pmb{p}$};
        
        \draw (1.0,0) rectangle (3.0, 2);
        \node at (2.0, 1.6) {SN-MNN};
        \node at (1.9, .7) {
        \begin{tikzpicture}[x=1.5cm, y=1.5cm, >=stealth, scale=0.25, every node/.style={transform shape}, opacity=0.2]
        
        \foreach \m [count=\i] in {1, memory, missing, 4, memory}
            \node [every neuron/.try, neuron \m/.try](input-\m-\i) at (1.0, 2.5-\i/2) {};
        
        \foreach \m [count=\i] in {1, memory, 3, memory, missing, 6, memory}
        	\node [every neuron/.try, neuron \m/.try](hidden-\m-\i) at (2.4, 2.8-\i/2) {};
        
        \foreach \i in {1, 2}
        	\node [every neuron/.try](output-\i) at (3.8,2.5-\i) {};
        
        \foreach \i in {1, 2}
        	\node [neuron memory/.try](output-memory-\i) at (3.8, 2.0-\i) {};
        
        \foreach \i in {1, 2}
        	\draw[->] (output-\i) -- (output-memory-\i);
        	
        \foreach \i in {1, 4}
        	\foreach \j in {1, 3, 6}
        		\draw [->] (input-\i-\i) -- (hidden-\j-\j);

        \foreach \i in {1, 3, 6}
        	\foreach \j in {1, 2}
        		\draw [->] (hidden-\i-\i) -- (output-\j);
        		
	\end{tikzpicture}
	};
	
	\draw (0.0, 3.5) rectangle (2.5, 5);
	\node at (1.25, 4.7) {Real World};
	\node at (1.25, 4.3) {UAV-Dynamics};
	\node at (1.25, 3.9) {$\dot{\pmb{x}} = f(\pmb{x}, \pmb{\omega})$};
	
	\draw[->] (2.5, 4.25) -- (3, 4.25);
	
	\node at (4.7, 5.3) {Local Odometry System};
	\draw (3.0, 3.5) rectangle (6.5, 5);
	\node at (3.75, 4.5) {Sensor};
	\node at (3.75, 4.2) {Suite};
    
    \draw[dashed] (4.75, 4) rectangle (5.75, 4.5);
    \node at (5.25, 4.25) {EKF};
    
    \draw[->] (3.0, 1) -- (5.25, 1) -- (5.25, 1.5);
    \node at (3.2, 1.3) {$\hat{\pmb{y}}$};
    
    \draw (4.5,1.5) rectangle (6.0, 2.5);
    \draw[->] (5.25, 2.5) -- (5.25, 4);
    \node at (5.25, 2.2) {ENU to};
    \node at (5.25, 1.8) {Geodetic};
    \node at (5.5, 2.8) {$\pmb{\zeta}$};
    
    \draw[->] (5.75, 4.25) -- (7.0, 4.25);
    \node at (7.2, 4.25) {$\hat{\pmb{x}}$};
    
    \end{tikzpicture}
    \caption{Figure illustrates the state fusion replay process. Based on the RPM input, the orientation of the UAV, and the previous position of the UAV, the trained SN-MNN predicts the position of the UAV, which is provided to the onboard EKF as GPS coordinates. The Sensor suite consists of peripheral sensors like IMU, compass, magnetometer etc., and the state fusion process is performed by the EKF.} \label{fig:block_diag}
\end{figure}
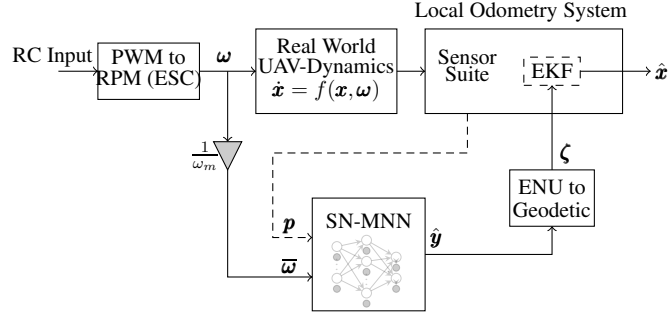
\begin{figure}[!htb]
        \includegraphics[scale=1, height=5.3cm]{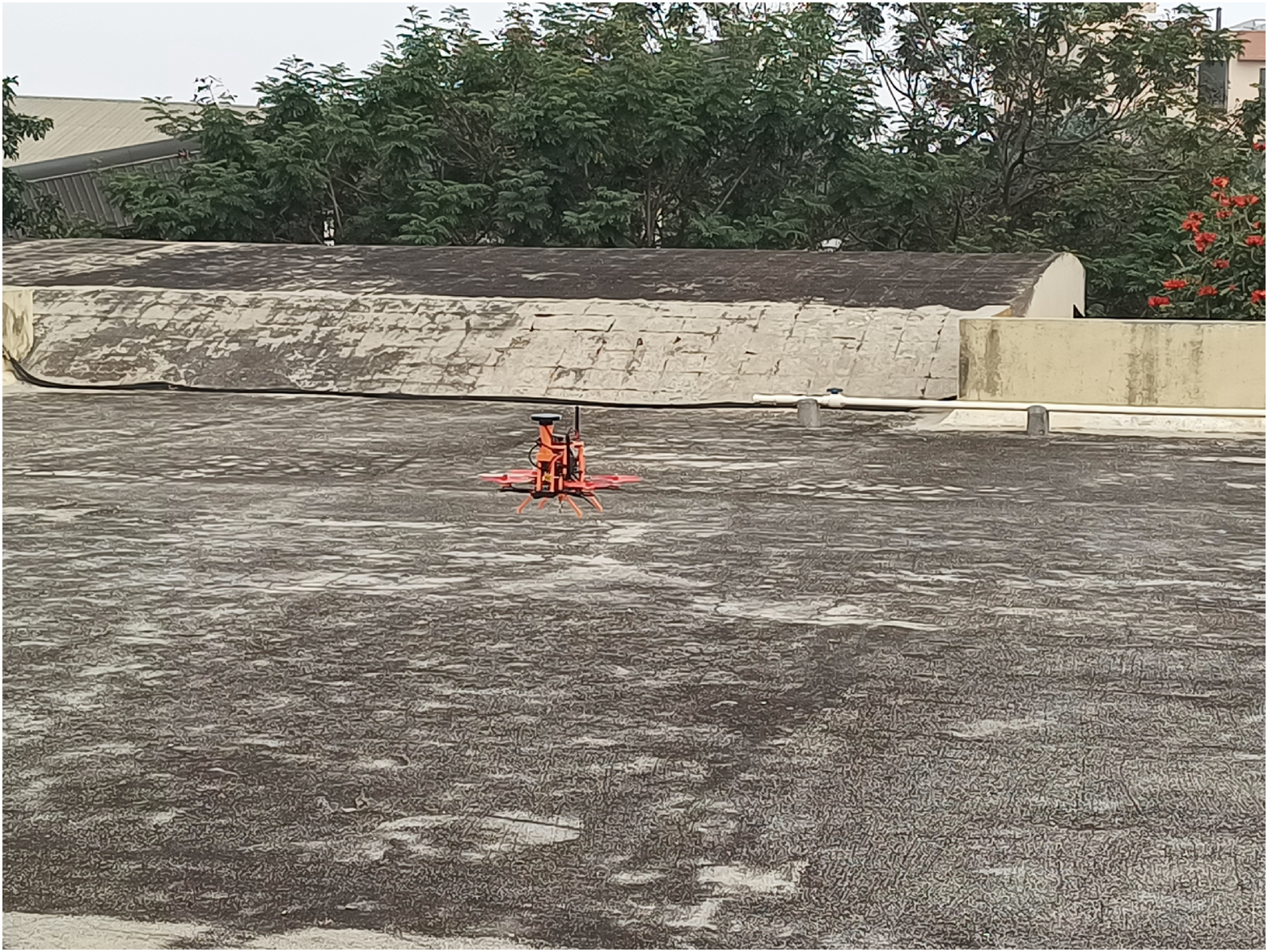}
        \hspace{0.3cm}
        \includegraphics[scale=1, height=5.3cm]{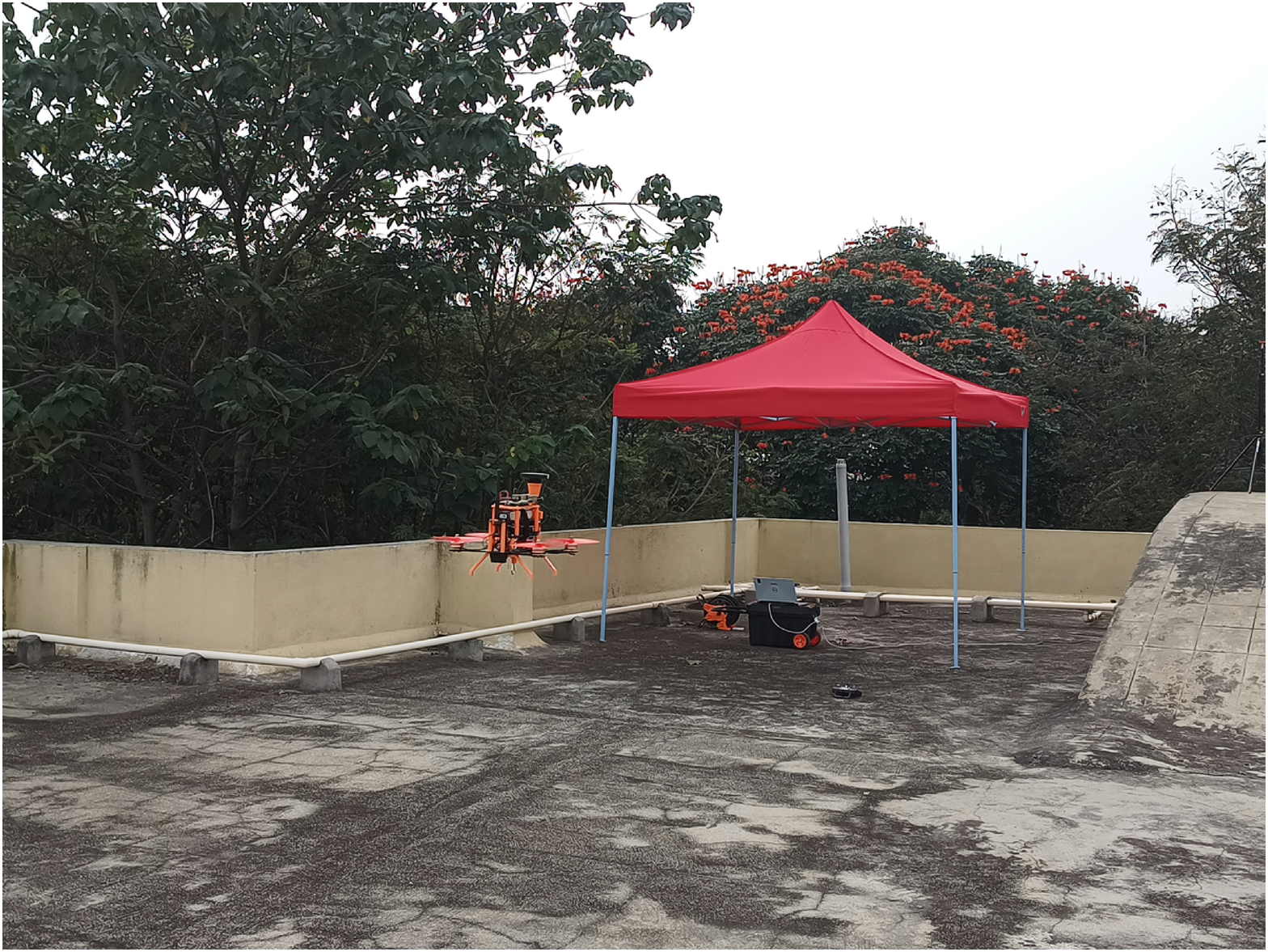}
    \caption{Figure on the left shows a snapshot of the UAV flight while collecting data. Figure on the right shows a snapshot of the RTK-GPS facility.}
    \label{fig:MOCAP_UAV}
\end{figure}
The ENU coordinates must first be converted to the ECEF (Earth-Center, Earth-Fixed) coordinate system before converting it to the GPS coordinates. Detailed information on converting the ENU coordinates to the GPS coordinates (via ECEF coordinates) is provided in the supplementary material \cite{nish2022supp_mat}. However, the following equations summarize the conversion to GPS coordinates:
\begin{equation}
\resizebox{.97\hsize}{!}{$\boxed{
    \varphi = 2tan^{-1}\left( \frac{Z}{D + \sqrt{D^2 + Z^2}} \right); \
    \lambda = 2tan^{-1}\left( \frac{Y}{X + \sqrt{X^2 + Y^2}} \right); \
    \mathfrak{z} = \frac{\kappa + e^2 - 1}{\kappa}\left( \sqrt{D^2 + Z^2} \right)
}$}\label{eq:ecef_geo}
\end{equation}
where the parameters $\kappa$ and $D$ can be calculated by the set of equations given in the supplementary material \cite{nish2022supp_mat}. This process is briefly summarized in Fig. \ref{fig:block_diag}.

\section{Experimental Results}
\label{sec:section_3}
First, this section presents the RTK GPS setup and the micro-UAV hardware configuration. Next, the experimental flight data collection for training the SN-MNN is discussed. Finally, the performance of SN-MNN prediction, flight evaluation and comparative study results are presented. 

\subsection{RTK GPS setup and UAV hardware}
The experimental setup consists of an outdoor RTK GPS facility (\texttt{Here+} RTK Base with \texttt{Here3} RTK GPS) with base station survey-in accuracy of 1m (relative accuracy of 10cm) and a custom-built micro-UAV (generic 250 racer frame) with the \texttt{Pixhawk 4} Flight Controller that runs on \texttt{PX4} firmware. The RTK facility and a micro-UAV operating in the facility are shown in Fig. \ref{fig:MOCAP_UAV}. The drone weighs about 1.1 Kg, with a RaspberryPi 4 onboard computer. The onboard sensors include an accelerometer, gyroscope, barometer and a compass, all present inside the flight controller. In addition to this, the Velox V2 1950KV T-motors are used with a 5-inch 3 blade propeller configuration. The data is collected on a typical day with mild steady wind influences.

\subsection{Experimental Flight Data Collection and Processing}
For training the network, experimental data is collected from the test facility. The states of the UAV and the $4$ rotor rpm are logged for multiple flights. The \texttt{DShot} protocol is used by the \texttt{F55A} electronic speed controller (ESC), which measures the rotor RPM values based on the back-EMF from the motor. The flight data mainly consists of random trajectories performed manually and certain square and circular trajectories performed autonomously by the UAV. These flights ensure that all possible UAV configuration in its state-space are captured.

Next, a common \textit{sampling frequency} is chosen to sample the data corresponding to different sensors (IMU, Barometer, Compass, RPM data and the RTK GPS position information). The RPM values are normalized based on the motor's maximum RPM. The final data consists of $12$ columns: one common timestamp for all other columns, $4$ normalized rotor rpm (from ESC), position, and quaternion orientation of the UAV. The entire data is then split into training and testing data in the ratio of 3:2.

\subsection{SN-MNN Prediction Performance Evaluation}

\begin{figure}
    \centering
    \includegraphics[scale=0.4]{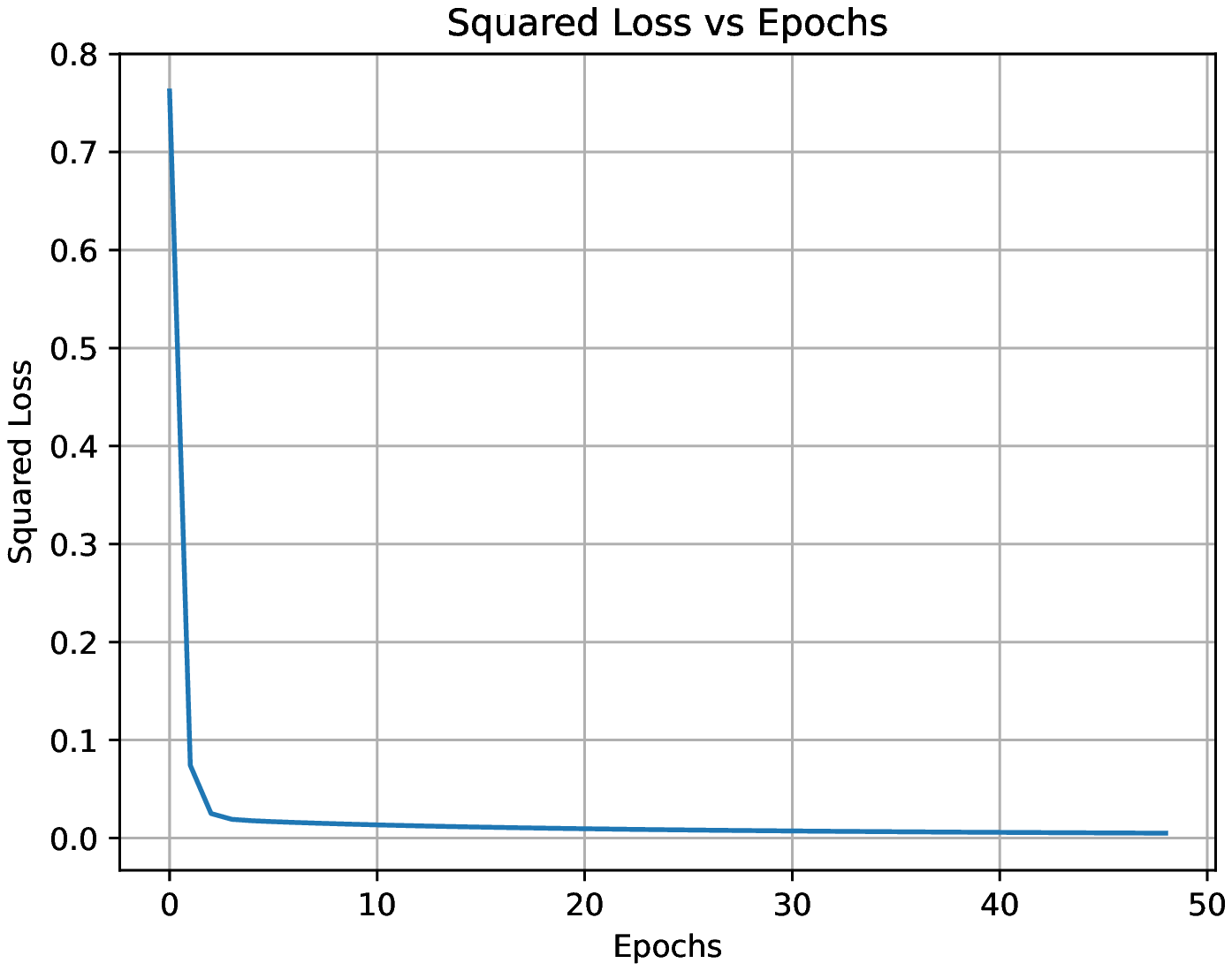}
    \hspace{0.5cm}
    \includegraphics[scale=0.4]{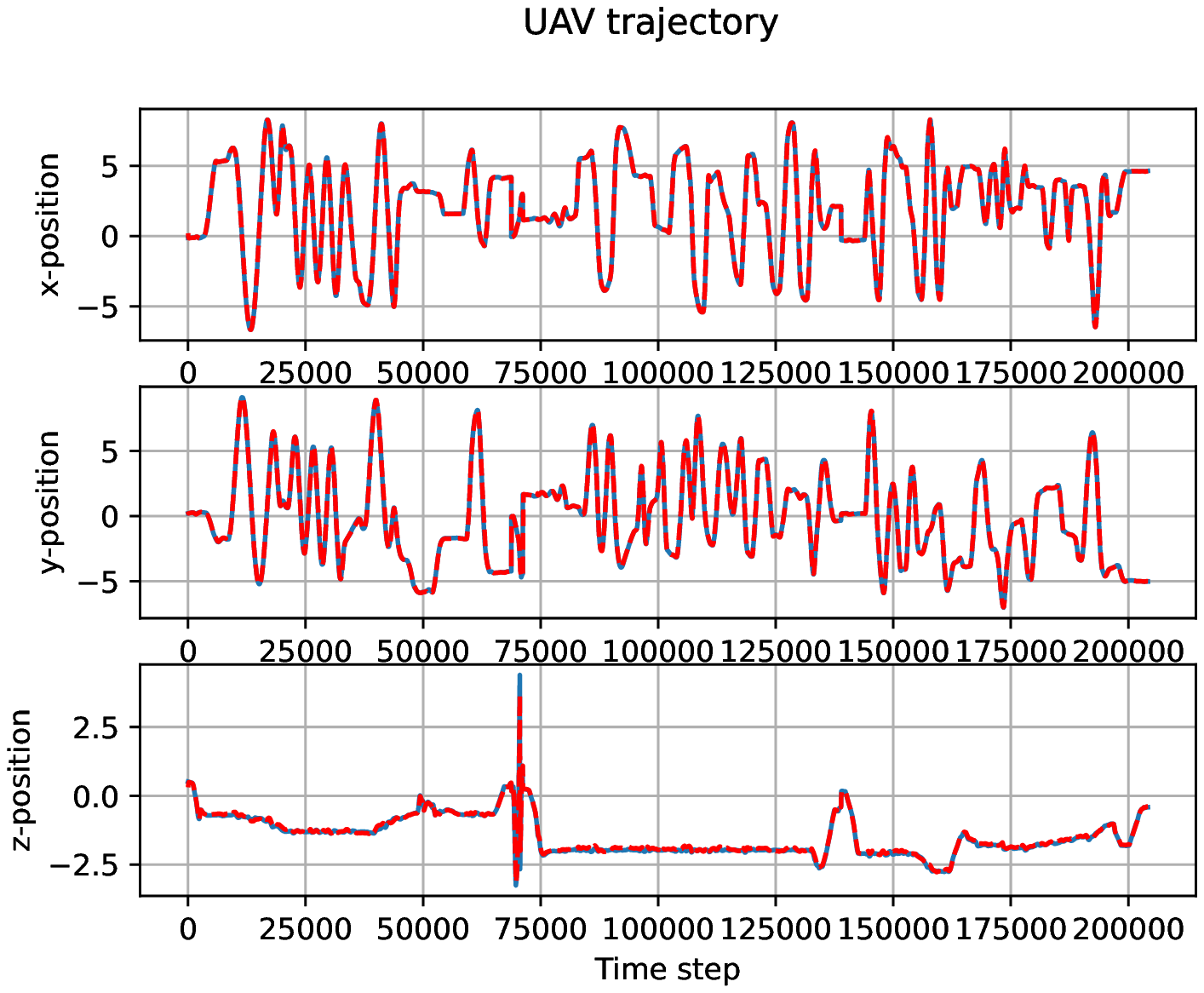}

    \includegraphics[scale=0.4]{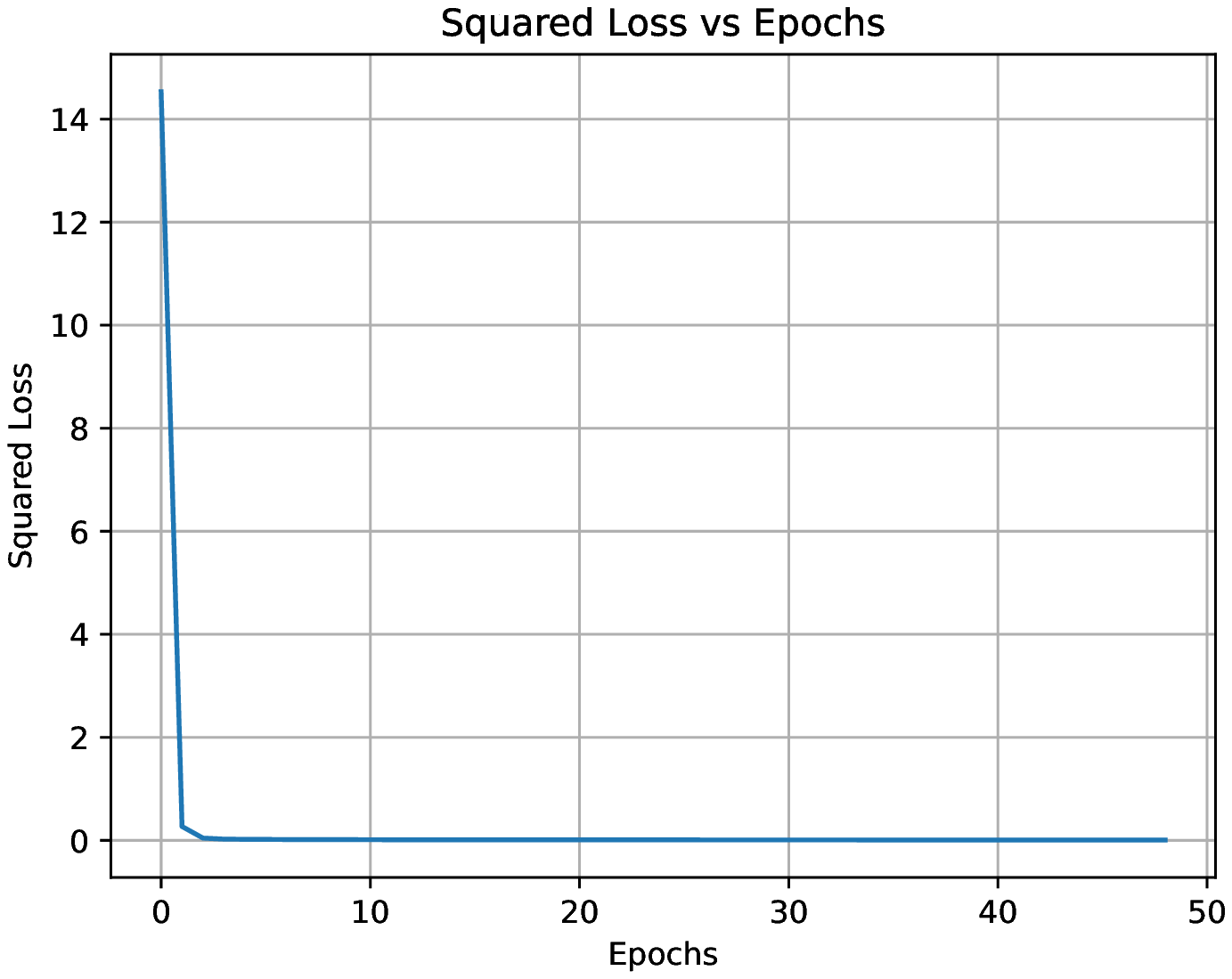}
    \hspace{0.5cm}
    \includegraphics[scale=0.4]{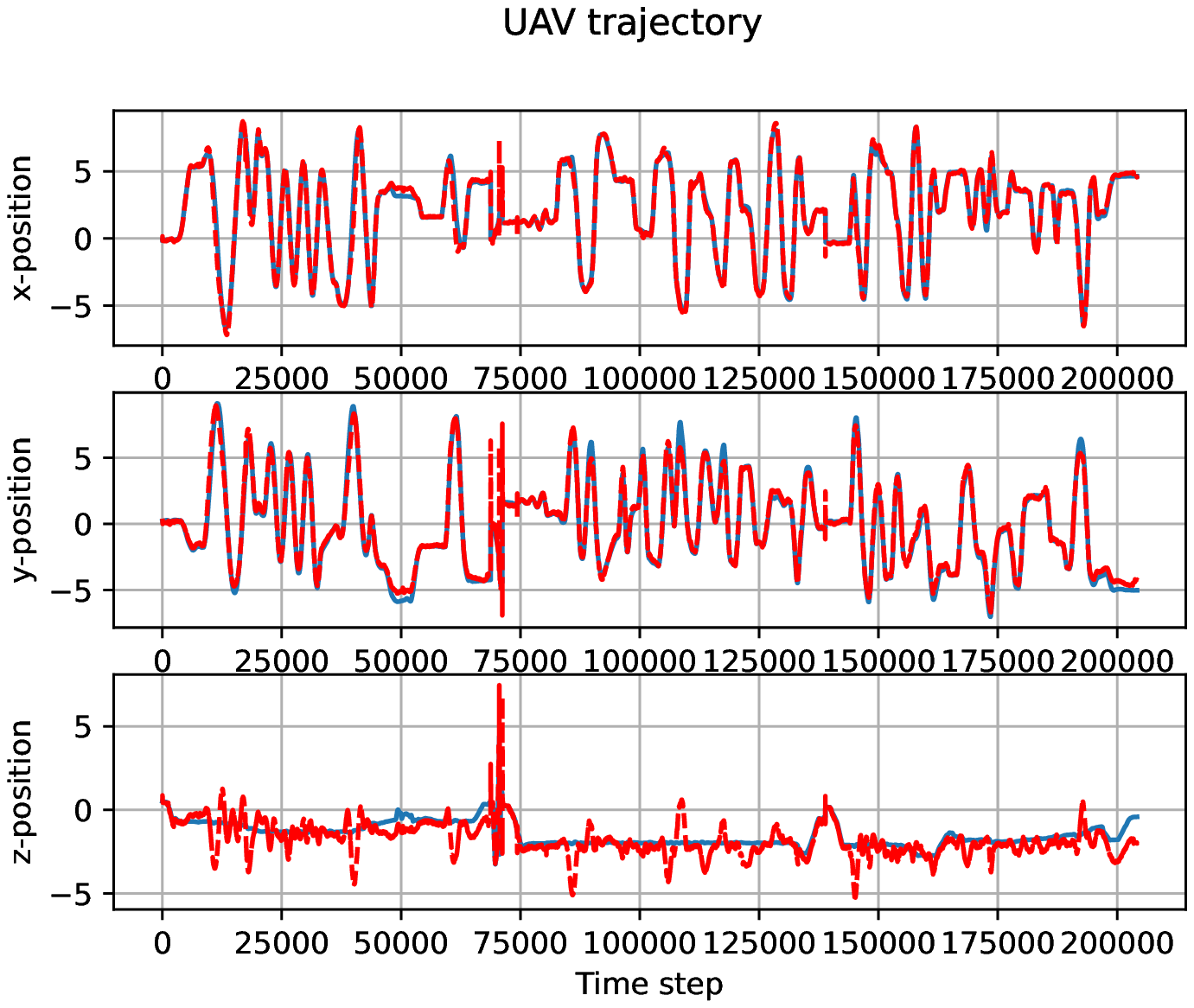}
\caption{Figure on the left shows the variation of the training loss vs epoch number. Figure on the right shows the prediction of the SN-MNN for the entire testing data set. The figures in the second row shows the performance of the network without spectral normalization.}
\label{fig:mnn_loss_pred}
\end{figure}

The SN-MNN is trained to predict the one-step-ahead position of the UAV, with the current UAV position $\pmb{y}_k\in\mathbb{R}^3$, orientation (Euler angles) $\pmb{q}_k\in\mathbb{R}^4$ and the normalized rotor RPM $\overline{\pmb{\omega}}_k\in\mathbb{R}^4$ as the input to the network. Hence, the selected architecture of SN-MNN is: $11$ input neurons, $100$ hidden neurons and $3$ output neurons. The network neurons in the hidden layer uses $tanh(.)$ activation and the network neurons in output layer employ linear activation function. The network is implemented in \texttt{Python} using \texttt{Numpy} Library. The network is trained for a total of $50$ epochs. The variation of the squared loss during training is shown in Fig. \ref{fig:mnn_loss_pred} for with and without spectral normalization. 

The output of the network for the entire test data (combined into one) is shown in Fig. \ref{fig:mnn_loss_pred}. The RMSE for the entire position prediction is about $5.84cm$, whereas, for the network without spectral normalization, the RMSE error is about $50cm$, thus justifying the requirement for a spectrally normalized network. It can be seen from Fig. \ref{fig:mnn_loss_pred} that the network has successfully learned the UAV dynamics accurately. The value of the Lipschitz constant $\gamma$ (here $\gamma=1$) of the network plays an important role in the stabilization of the network during the training process. It also determines "how fast" the network output can vary: For e.g. if the UAV is mostly hovering and making slow movements, the Lipschitz constant can be set to a low value. If the UAV is performing aggressive sharp maneuvers frequently, then the Lipschitz constant of the network must be set to a high value. In this paper, the trajectory data is collected for a UAV that mostly cruises and exhibits slow maneuvers.

\subsection{State Fusion and Experimental data evaluation}
Once the predicted position $\pmb{y}_k$ is obtained from the SN-MNN, the GPS geodetic vector $\pmb{\zeta}_k$ is calculated from Eq. \ref{eq:ecef_geo}, which is then given to the EKF for state fusion along with other state information like orientation, linear velocities (IMU) and UAV heading (compass). This is illustrated in Fig. \ref{fig:block_diag}. An instance of the \texttt{PX4-ECL} library is started on the onboard Raspberry-pi computer, and the state estimation process using the position predicted by the SN-MNN starts parallelly with the state-estimation process on the \texttt{PX4}-Autopilot that uses the GPS for the UAV flight. This is done to facilitate a real-time comparison of the position obtained (after state fusion) from the proposed algorithm vs. the position obtained (after state fusion) using an RTK GPS.


\begin{figure}[!htb]
    \centering
    \includegraphics[scale=0.4]{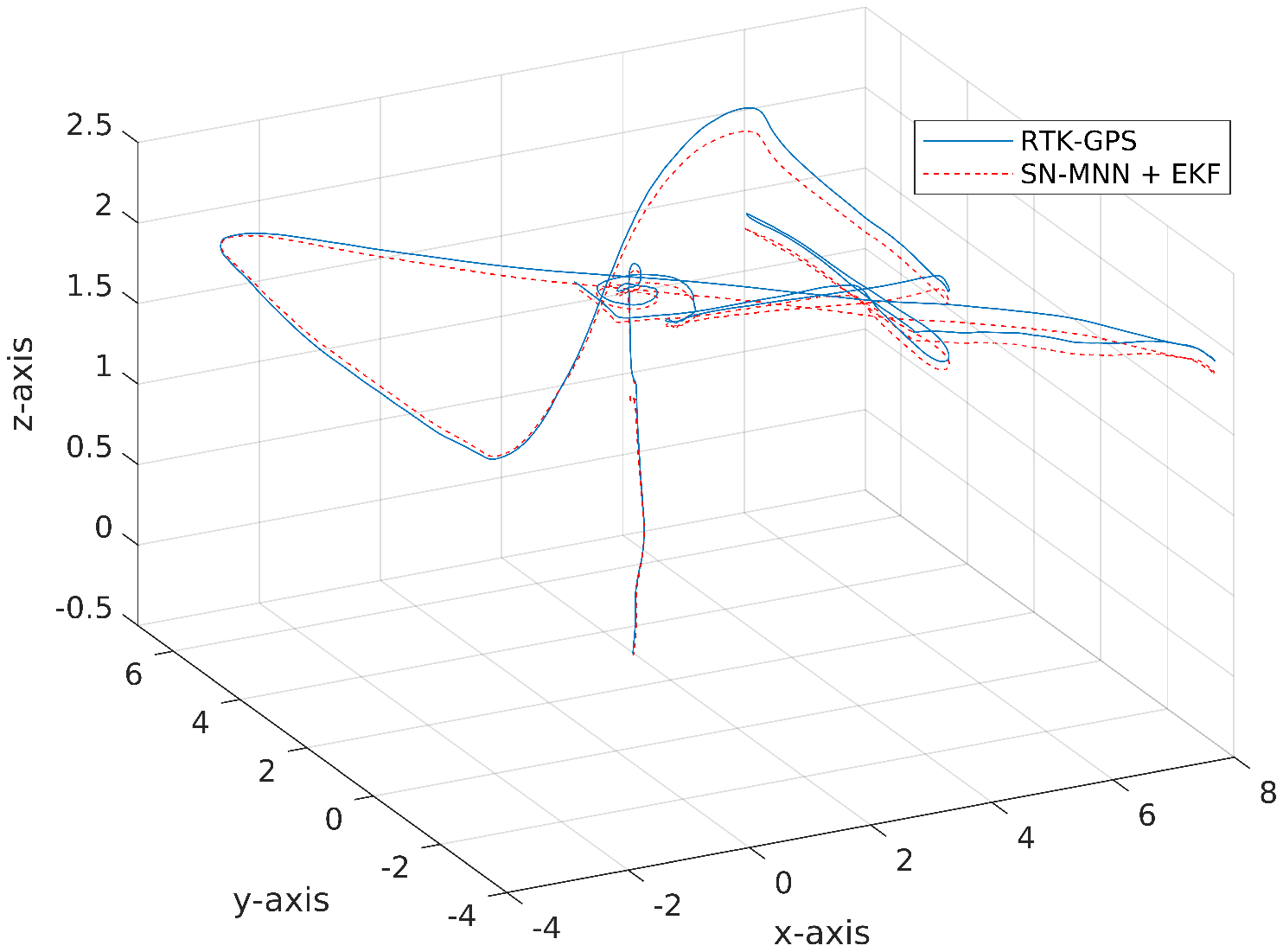}
    \hspace{0.85cm}
    \includegraphics[scale=0.4]{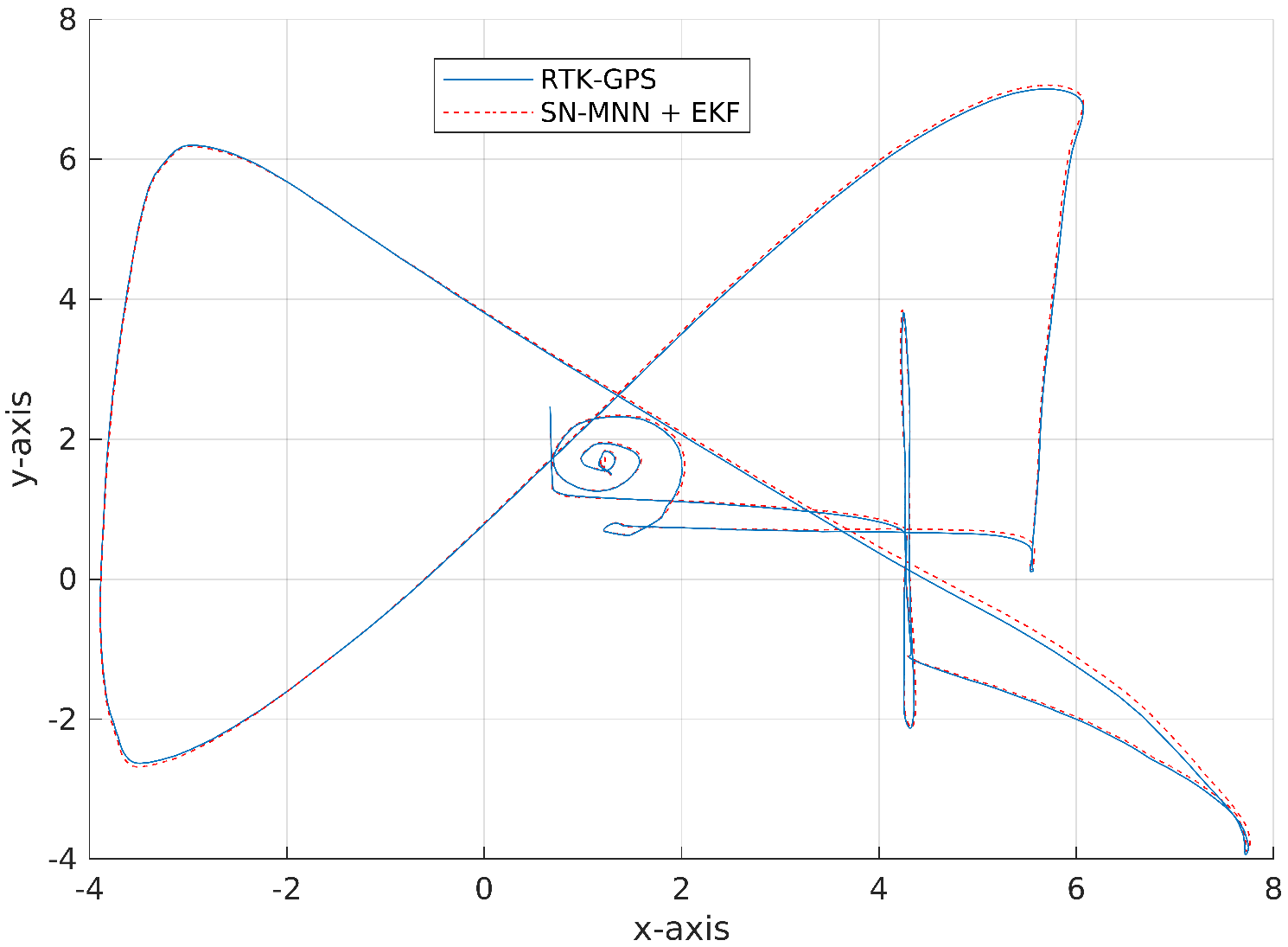}

    \includegraphics[scale=0.4]{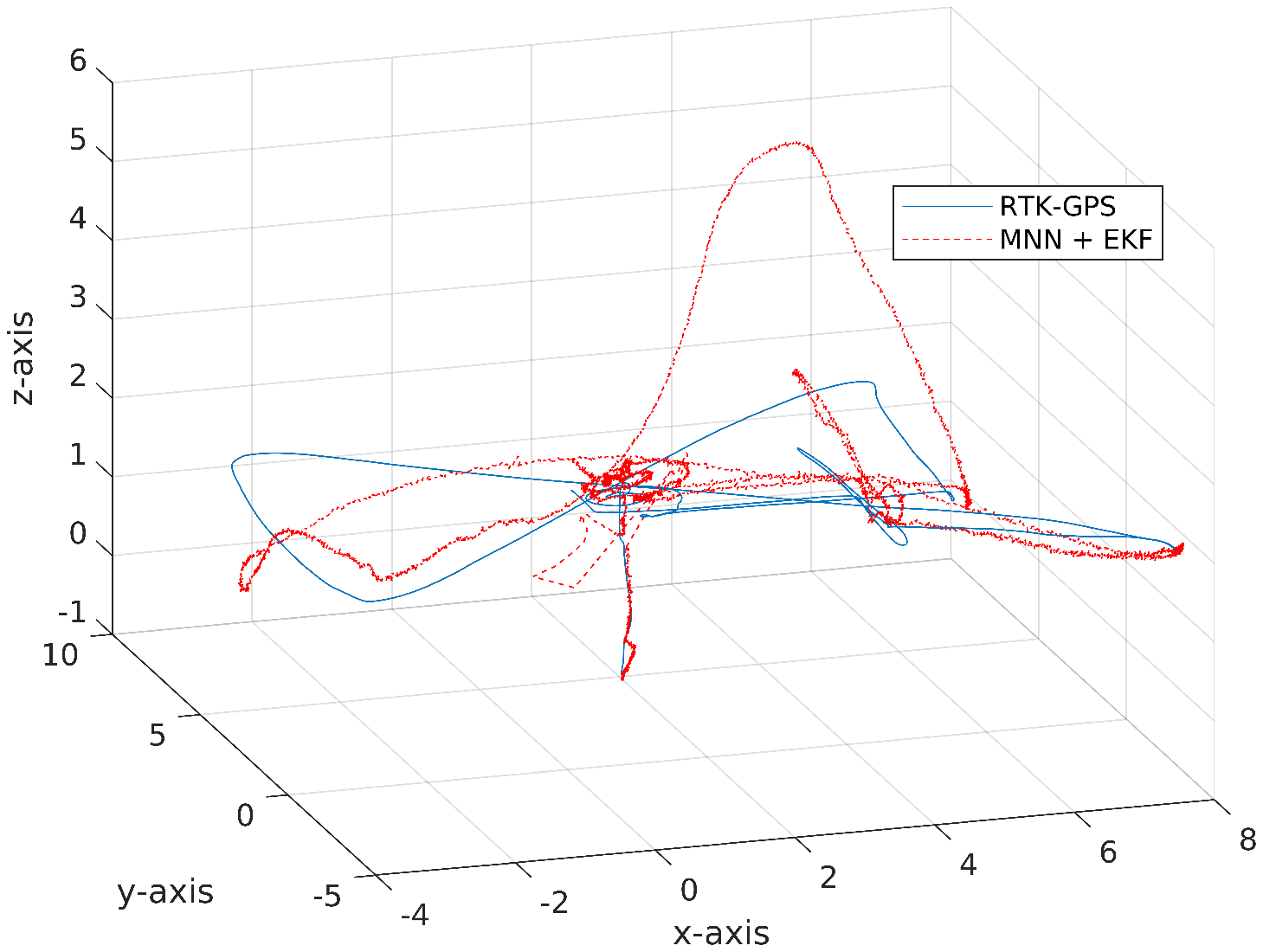}
    \hspace{0.85cm}
    \includegraphics[scale=0.4]{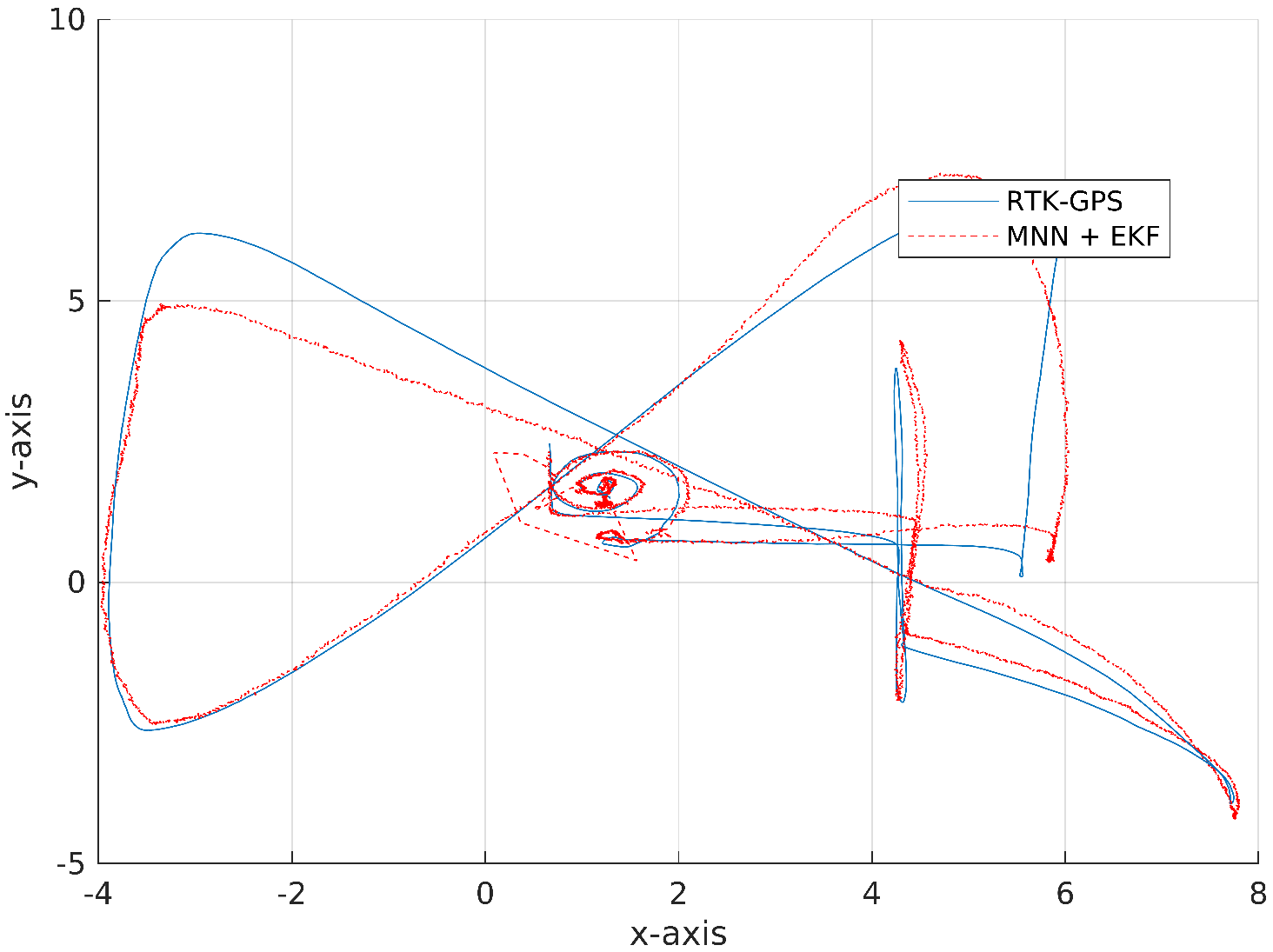}
    \caption{Figures illustrate the performance of the proposed algorithm with and without \textit{spectral normalization} for the first sample test trajectory. For convenience, the top views are shown on the right, corresponding to their 3D trajectories. The second-row figures show the performance without the spectral normalization process.}
    \label{fig:test_1}
\end{figure}

\begin{figure}[!htb]
    \centering
    \includegraphics[scale=0.4]{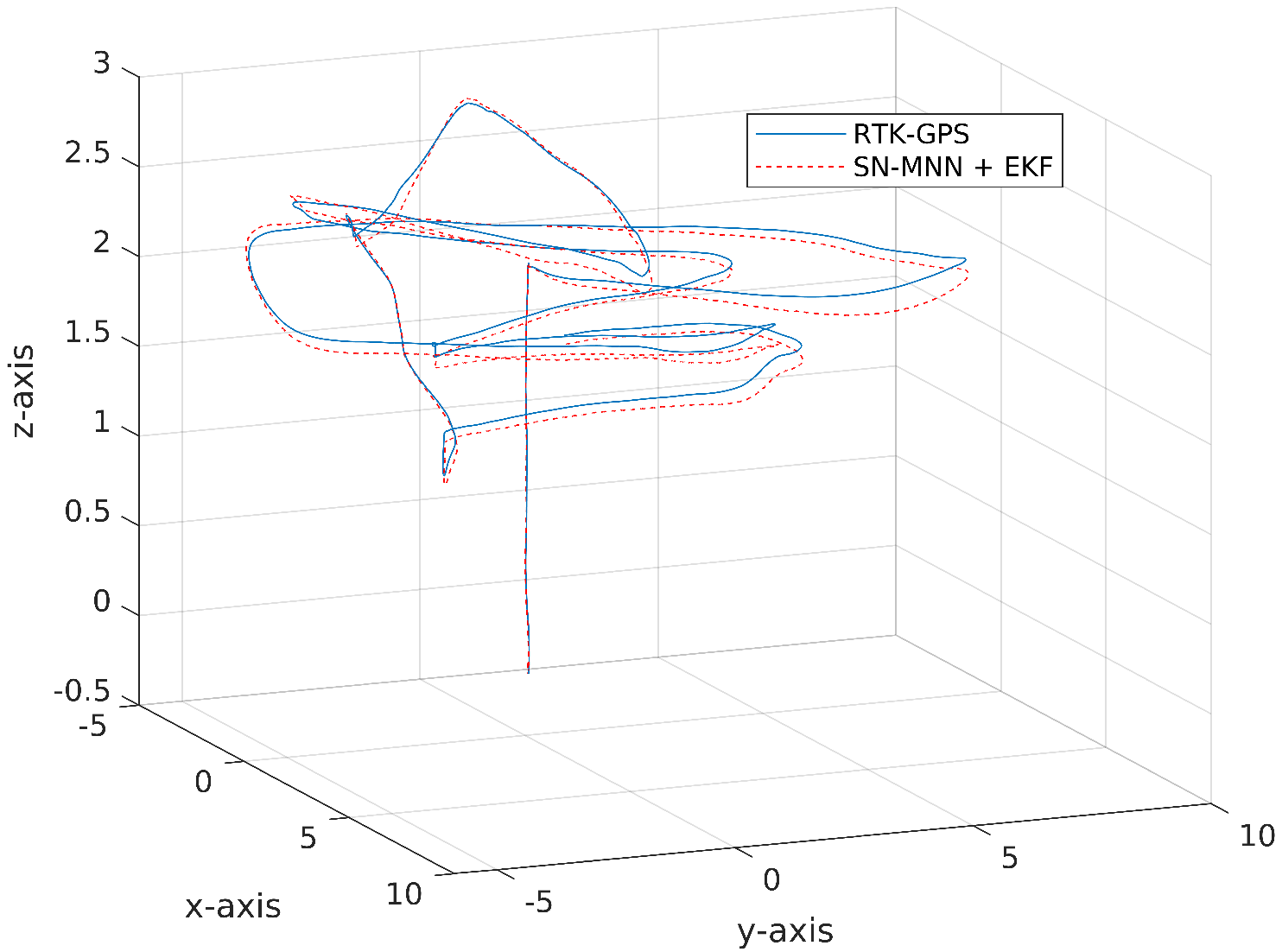}
    \hspace{0.85cm}
    \includegraphics[scale=0.4]{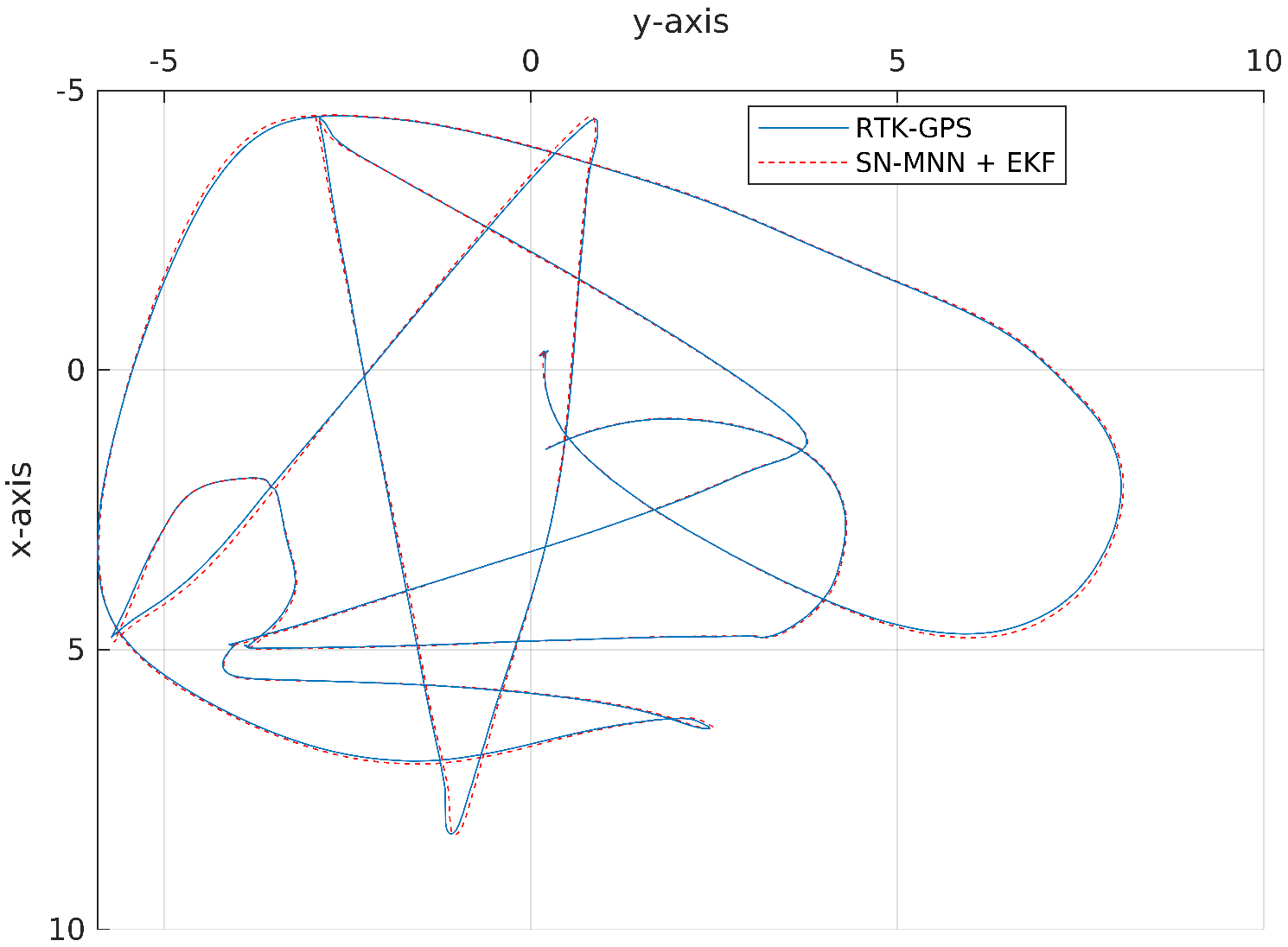}

    \includegraphics[scale=0.4]{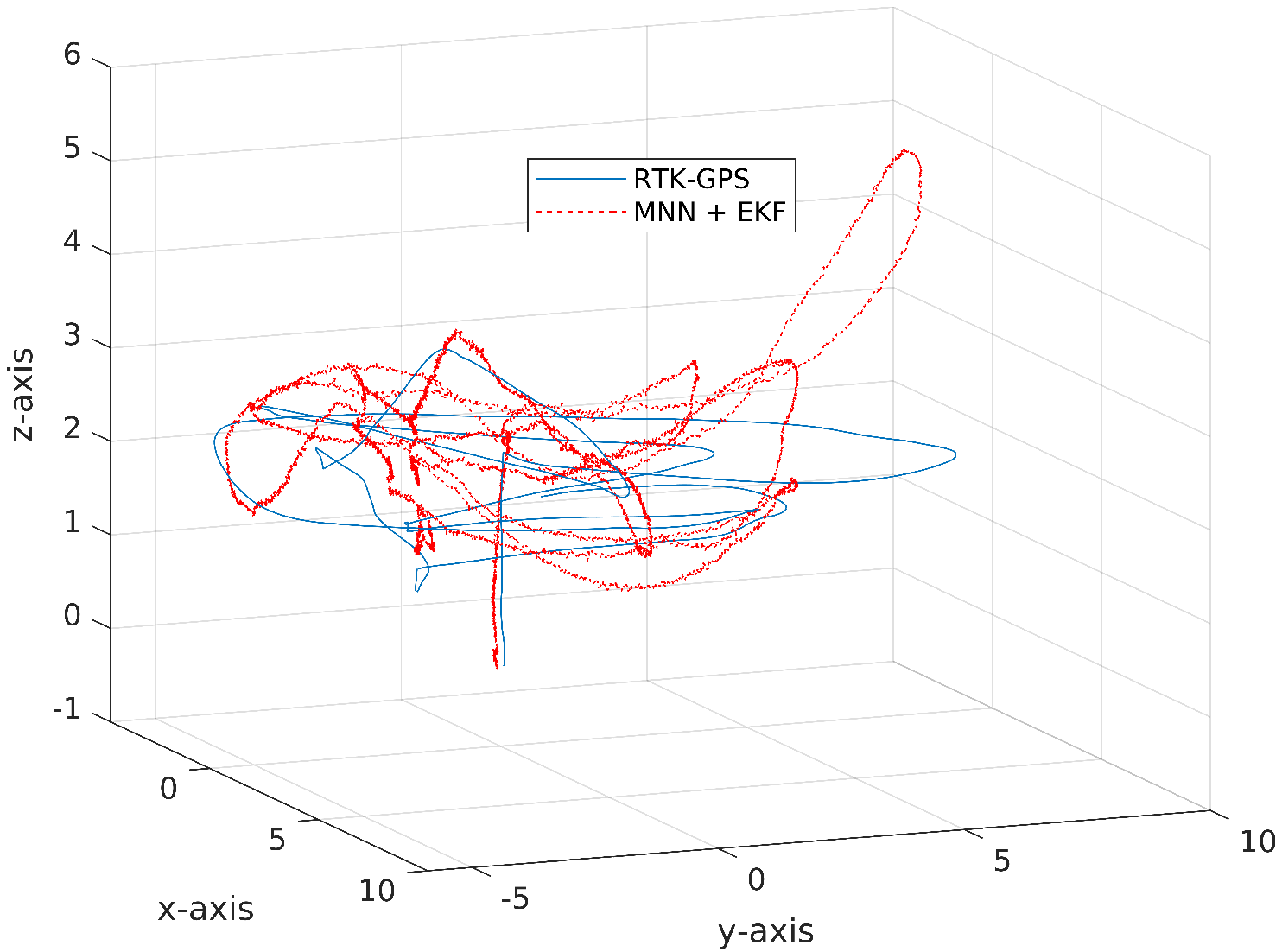}
    \hspace{0.85cm}
    \includegraphics[scale=0.4]{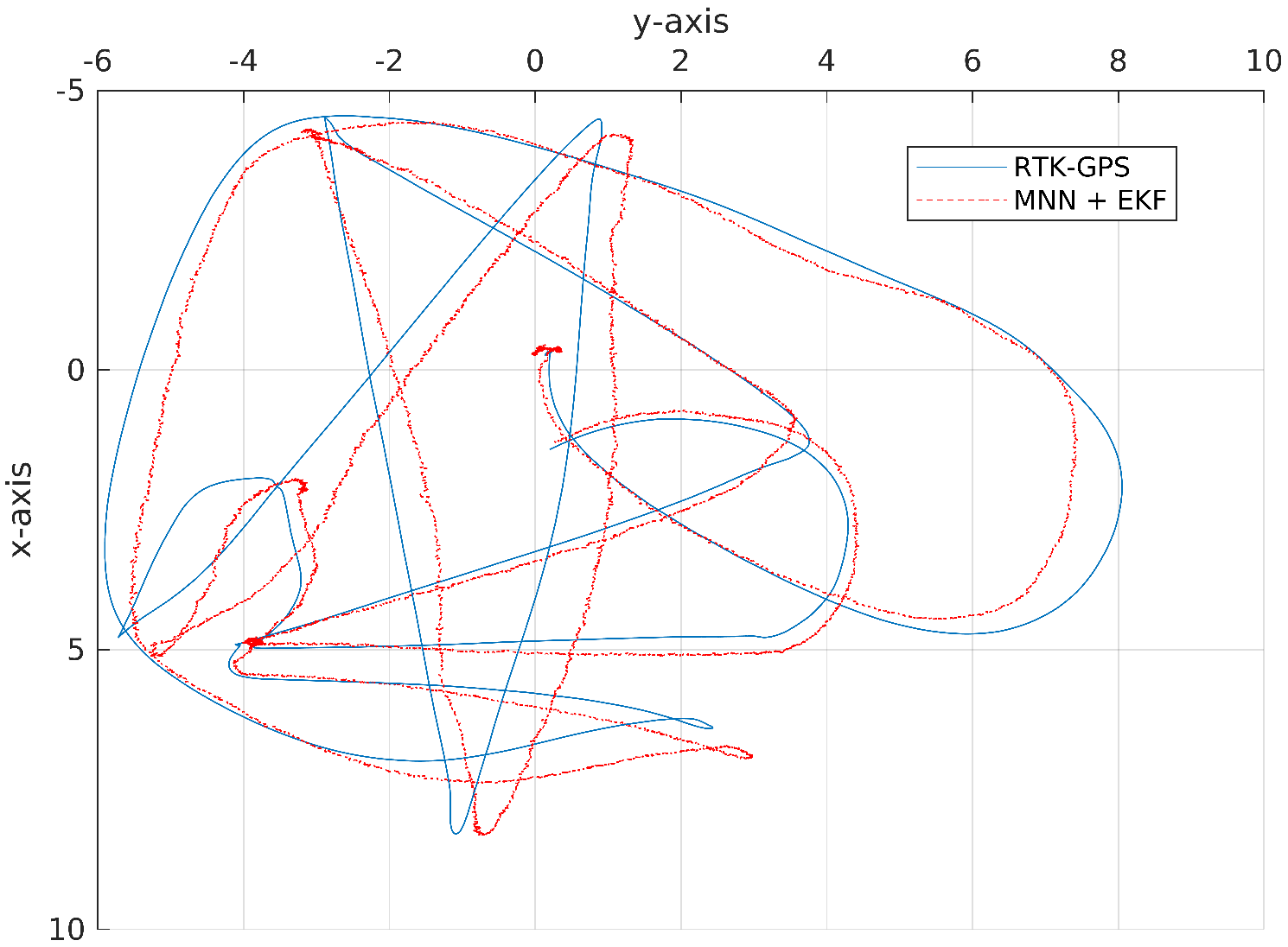}
    \caption{Figures illustrate the performance of the proposed algorithm with and without \textit{spectral normalization} for the second sample test trajectory. For convenience, the top views are shown on the right, corresponding to their 3D trajectories. The second-row figures show the performance without the spectral normalization process.}
    \label{fig:test_2}
\end{figure}

\begin{table}[!htb]
    \centering
    \caption{RMSE comparison}
    \label{tab:rmse}
    \begin{tabular}{ccccc} \toprule
          & VINS-Mono & VIO & MNN + EKF & \textbf{SN-MNN + EKF} \\ \midrule
         RMSE (m) & 0.18 & 0.13 & 0.542 & $\mathbf{0.05953}$  \\ \bottomrule
    \end{tabular}
\end{table}
The outputs for two sample test trajectories are shown in Fig. \ref{fig:test_1} and Fig. \ref{fig:test_2} after the \texttt{PX4-ECL} state fusion is performed. It must be noted that the GPS fusion only affects the position and linear velocity of the UAV, not its rotational components. The RMSE between the estimated positions and the actual positions of the UAV for the course of the entire test flight duration (all test trajectories combined) is around $6cm$, whereas, for the network without spectral normalization, the RMSE is around $55cm$. Further, Table 1 shows a comparison with two state-of-art methods, namely the Visual Inertial Navigation System (VINS) and the Visual Inertial Odometry (VIO). The RMSE reported by these methods is compared along with the proposed method in this work. It can be seen that the RMSE reduces by about $60\%$ for VINS-mono\citep{8421746} and about $40\%$ for VIO\citep{loianno2016visual}. Further, the algorithm is computationally light, as one has to implement only the trained feedforward SN-MNN model. This implementation can be done directly on the flight controller, requiring no additional onboard computer. Thus, from the above results, it can be concluded that the proposed algorithm can be used to estimate the UAV states from the rotor RPM reliably. 

\bibliography{references}
\nocite{*}

\end{document}